\documentclass[letterpaper, 10 pt, conference]{ieeeconf}
\IEEEoverridecommandlockouts                                                                         
\usepackage{amsmath,amssymb,amsfonts}
\usepackage{url}
\usepackage{hyperref}
\usepackage{cleveref}
\usepackage{color}
\usepackage{xspace}
\usepackage{mathtools}
\usepackage{graphicx}      
\usepackage{algorithm}
\usepackage{algorithmic}
\usepackage{wrapfig}

\let\labelindent\relax
\usepackage{enumitem}

\newcommand{\eps}{{\varepsilon}}										
\newcommand{\E}[2][{}]{\mathbb E_{#1}\left[#2\right]}					
\newcommand{\PP}[1]{\mathbb P\left[#1\right]}							
\newcommand{\ra}{\rightarrow}											
\newcommand{\la}{\leftarrow}											
\newcommand{\ie}{\unskip, i.\,e.,\xspace}								

\newcommand{\Z}{{\mathbb{Z}}}											
\newcommand{\R}{{\mathbb{R}}}											
\newcommand{\state}{s}													
\newcommand{\State}{S}													
\newcommand{\states}{\mathbb S}											
\newcommand{\action}{a}													
\newcommand{\Action}{A}													
\newcommand{\actions}{\mathbb A}										
\newcommand{\policy}{\pi}												
\newcommand{\policies}{\Pi}												
\newcommand{\transit}{p}												
\newcommand{\reward}{r}													
\newcommand{\Value}{v}													
\newcommand{\G}{\ensuremath{\mathbb{G}}}								
\newcommand{\K}{\ensuremath{\mathcal{K}}\xspace}						
\newcommand{\KL}{\ensuremath{\mathcal{KL}}\xspace}						

\newtheorem{dfn}{Definition}
\newtheorem{theorem}{Theorem}

\newtheorem{proposition}{Proposition}
\newtheorem{corollary}{Corollary}
\newcommand{\goaldist}[1][{}]{d_{\G#1}}	
\newcommand{\basepolicy}{\policy_{b}}
\newcommand{\fallbackpolicy}{\policy_{f}}
\newcommand{\calfwpolicy}{\policy_{t}}
\newcommand{\Valuebase}{\hat{\Value}^{\basepolicy}}
\newcommand{\subscript}[2]{$#1 _ #2$}

\newcommand{\calfwrepo}{{\footnotesize\url{https://github.com/aidagroup/calf-wrapper}}\xspace}
\newcommand{\whiteqedsymbol}{\(\square\)}
\newenvironment{whiteproof}[1][]%
{%
	\par\noindent\qquad\textit{Proof#1:}
}%
{%
	\hfill\whiteqedsymbol\par
}
\expandafter\def\expandafter\normalsize\expandafter{%
    \normalsize%
    \setlength\abovedisplayskip{3pt}%
    \setlength\belowdisplayskip{3pt}%
    \setlength\abovedisplayshortskip{3pt}%
    \setlength\belowdisplayshortskip{3pt}%
}
\title{\LARGE \bf
A universal policy wrapper with guarantees
}

\author{Anton Bolychev$^{1}$, Georgiy Malaniya$^{1}$, Grigory Yaremenko$^{1}$, Anastasia Krasnaya$^{1}$, Pavel Osinenko$^{1}$ 
\thanks{$^{1}$Skolkovo Institute of Science and Technology}%
\thanks{Corresponding author: P. Osinenko, email: {\tt\scriptsize p.osinenko@gmail.com}.
	First two authors contributed equally.}%
}

\begin{document}

\maketitle
\thispagestyle{empty}
\pagestyle{empty}

\begin{abstract}

	We introduce a universal policy wrapper for reinforcement learning agents that ensures formal goal-reaching guarantees.
	In contrast to standard reinforcement learning algorithms that excel in performance but lack rigorous safety assurances, our wrapper selectively switches between a high-performing \emph{base policy} — derived from any existing RL method — and a \emph{fallback policy} with known convergence properties.
	Base policy's value function supervises this switching process, determining when the fallback policy should override the base policy to ensure the system remains on a stable path.
	The analysis proves that our wrapper inherits the fallback policy's goal-reaching guarantees while preserving or improving upon the performance of the base policy.
	Notably, it operates without needing additional system knowledge or online constrained optimization, making it readily deployable across diverse reinforcement learning architectures and tasks.

\end{abstract}

\section{INTRODUCTION}

Reinforcement learning (RL) has proved to be a powerful framework for solving optimal control problems in complex, high-dimensional settings.
Most RL algorithms optimize control policies through direct interaction with a system (often called an environment), relying solely on reward signals observed at each step, without requiring explicit knowledge of system dynamics.
Its remarkable successes span applications such as robotic manipulation~\cite{Akkaya2019Solvingrubiks, Kumar2016Optimalcontrol}
board games like Go, chess, and shogi~\cite{Silver2018generalreinfor}, and complex real-time video games~\cite{Vinyals2019Grandmasterlev}.
However, commonly used algorithms such as
Twin-Delayed Deep Deterministic Policy Gradient~\cite{Fujimoto2018}, Soft Actor-Critic~\cite{Haarnoja2019}, and Proximal Policy Optimization~\cite{Schulman2017ProximalPolicy} depend on deep neural networks and data-driven exploration, complicating stability analysis and offering no rigorous theoretical stability guarantees, which are particularly critical for safety-sensitive applications.

However, several RL approaches do seek to ensure stability under certain constraints.
Our work builds upon and extends~\cite{Osinenko2024CriticLyapunov}, which introduced \emph{Critic as a Lyapunov Function (CALF)} --- an RL method ensuring stability via constrained optimization at each step.
\textit{Our key insight:} Rather than solving these optimization problems, we can \emph{wrap} a ready-to-use RL policy with an action filter that invokes a \emph{fallback} controller when needed, guided by a critic value.

We refer to the RL policy as the \emph{base policy}; it pursues high practical performance but may lack formal assurances.
Meanwhile, the \emph{fallback policy} need not be reward-optimal.
It is a nominal, generally problem-agnostic policy, equipped with statistical guarantees of driving the system to a goal set $\G$, which is specified by the problem formulation and may, for example, be a compact neighborhood of the origin or the origin itself.
We refer to our approach as the \emph{CALF-Wrapper}.

\subsection{Key Contributions}
\begin{itemize}[leftmargin=1.5em]
	\item A novel runtime wrapper that can be applied to any off-the-shelf RL policy, converting it into a policy with formal goal-reaching guarantees (see \Cref{sec:approach});
	\item Theoretical goal-reaching guarantees with rigorous mathematical proofs (Theorems~\ref{thm:etagoalreaching} and~\ref{thm:unform_goal_reaching});
	\item Open-source implementation \calfwrepo designed for seamless integration with existing RL frameworks (e.g.~\cite{stable-baselines3,huang2022cleanrl});
	\item Empirical evidence of maintained or improved performance compared to both base and fallback policies (see Sections~\ref{sec:environments} and~\ref{sec:experiments}).
\end{itemize}

\subsection{Related Work on Safe Reinforcement Learning}\label{subsec:related_work}

Safe reinforcement learning incorporates constraints into conventional RL, ensuring system safety while optimizing performance. Our \textit{CALF-Wrapper} provides stability guarantees by leveraging base and fallback policies during runtime without requiring optimization or additional knowledge. In contrast, existing approaches typically depend on continuous optimization, access to system dynamics, or access to Lyapunov/barrier functions that are often difficult to calculate. These include:

\textit{Lyapunov-based methods} enforce safety by requiring a decrease in a Lyapunov function during training, thereby guaranteeing stability along system trajectories~\cite{perkins2002lyapunov, chow2018lyapunov}.
These approaches typically construct Lyapunov functions online or assume prior access to them with additional solving constrained optimization sub-problems for policy updates.

\textit{Shielding and filtering approaches} Safety architectures—first proposed in~\cite{seto1998}—enable safe online controller tuning by switching between a dedicated safety controller and an experimental one.
This concept has inspired frameworks based on control barrier functions~\cite{ames2019}, which is still challenging to compute.
At runtime, a barrier function or safety supervisor~\cite{alshiekh2018shielding} intercepts potentially unsafe actions, effectively filtering them out~\cite{dalal2018continuous, cheng2019endtoend}.
The approaches require system-dynamics knowledge and often involve solving online corrections via optimization to filter unsafe actions.

\textit{Constrained policy optimization techniques}, such as Constrained Policy Optimization~\cite{achiam2017cpo} and Reward Constrained Policy Optimization~\cite{tessler2019rcpo}, embed safety constraints into policy updates during training, requiring constrained optimization at each step.
Recent methods like Constrained Update Projection Approach~\cite{yang2022cup} enhance these guarantees by projecting updates onto safe regions.

\section{PROBLEM STATEMENT}\label{sec_problem}

This section establishes the formal reinforcement learning framework and mathematical foundations necessary to rigorously define our goal-reaching guarantees.

\subsection{Markov Decision Process}

We define a Markov decision process (MDP) as:
\begin{equation}
	\label{eq:mdp_definition}
	\mathcal{M}
	\;=\;\bigl(
	\states,\;\actions,\;\transit_0,\;\transit,\;\reward
	\bigr)
\end{equation}

\noindent
where $\states$ is the state space (a finite-dimensional Banach space), $\actions$ is the action space (a compact topological space), $\transit_0: \states \to \R$ is the initial state probability density function, $\transit(\bullet \mid \state,\action)$ is the transition probability density \ie the probability density of the next state $\State' \sim \transit(\bullet \mid \state, \action)$ given the current state $\state \in \states$ and action $\action \in \actions$, and $\reward(\state,\action) \in \mathbb{R}$ is the immediate reward of taking action $\action \in \actions$ in state $\state \in \states$. For theoretical guarantees, we assume there exists an upper semi-continuous function $\bar{p} : \states \times \actions \to \mathbb{R}_{\geq 0}$ such that for any $\state \in \states$ and $\action \in \actions$, the next state $\State'$ sampled from $\transit(\bullet \mid \state, \action)$ satisfies $\PP{\|\State'\| \leq \bar{p}(\state, \action)} = 1$, effectively bounding the system's one-step transition magnitude.

\subsection{Policies and Value Functions}

We consider two types of policies:

(1) \textit{Stationary policies} $\Pi_{\mathrm{stat}}$. For any stationary policy $\pi \in \Pi_{\mathrm{stat}}$, $\pi(\action \mid \state)$ represents the probability density of selecting action $\action$ in state $\state$.

(2) \textit{Non-stationary policies} $\policies_{\mathrm{nstat}}$. For any non-stationary policy $\pi \in \policies_{\mathrm{nstat}}$, $\pi(\action \mid \state, t)$ represents the probability density of selecting action $\action$ in state $\state$ at time $t$.
By definition, \(\Pi_{\mathrm{stat}} \subset \Pi_{\mathrm{nstat}}\), since any
stationary policy can be viewed as a special case of a non-stationary policy that
remains constant across time steps.

For a discount factor $\gamma \in [0,1)$ and a policy $\policy \in \policies_{\mathrm{nstat}}$, the \emph{value function} $\Value^\policy : \states \to \mathbb{R}$ is:
\begin{equation*}
		\label{eq:value_function}
		\Value^\policy(\state)
		=
		\E{
			\sum_{t=0}^{\infty}
			\gamma^t\,
			\reward(\State_t,\Action_t)
			\Bigm|
			\State_0 = \state,\,
			\Action_t \sim \policy(\bullet \mid \State_t)
		}.
\end{equation*}

The standard RL objective is to maximize expected discounted rewards:
\begin{equation}
	\label{eq:rl_objective}
	\E{\Value^\policy(\State_0)} = \E{\sum_{t = 0}^{\infty}\gamma^t \reward(\State_t, \Action_t)} \ra \max_{\policy\in\policies_{\mathrm{nstat}}},
\end{equation}
where $\State_0 \sim \transit_0(\bullet)$, $\Action_t \sim \policy(\bullet \mid \State_t, t)$, and $\State_{t+1} \sim \transit(\bullet \mid \State_t,\Action_t)$ for $t \ge 0$.

For a given policy $\policy$ and initial state $\state_0 \in \states$, we denote the state at time $t$ by $\State_t^{\policy}(\state_0)$.

\subsection{Goal Set and Goal Reaching Property}

While traditional RL focuses on reward maximization, many control problems—particularly in safety-critical applications—require driving the system to a specific target region. We formalize this requirement as follows:

\begin{dfn}[Goal Set]\label{dfn:goal_set}
	A \emph{goal set} $\G \subset \states$ is a compact neighborhood of the origin representing a target region where we want the system to eventually reach and remain.
\end{dfn}

The distance from state $\state$ to the goal set is defined as:
$
	\goaldist(\state)
	:=
	\inf_{\state' \in \G} \,\|\state - \state'\|
$,
where $\goaldist(\state)=0$ iff $\state \in \G$.

We now introduce two key properties characterizing policies that reliably drive the system to the goal:

\begin{dfn}[$\eps$-Improbable Goal Reaching Property]\label{dfn:eta_improbable}
	A policy $\policy \in \policies_{\mathrm{nstat}}$ satisfies the $\eps$-\emph{improbable goal reaching property} for some $\eps \in [0,1)$ if, for all initial states $\state_0 \in \states$:
	\[
		\label{eqn_introstab}
		\PP{
			\goaldist(\State_t^{\policy}(\state_0))
			\,\xrightarrow{t \to \infty}\,
			0
		}
		\ge
		1 - \eps.
	\]
	In other words, the system reaches the goal set in the limit with probability at least \(1-\eps\).
\end{dfn}

For more precise control over convergence rate, we introduce a stronger property:

\begin{dfn}[Uniform $\eps$-Improbable Goal Reaching]\label{dfn:uniform_eta_improbable}
	A policy $\policy \in \policies_{\mathrm{nstat}}$ satisfies the \emph{uniform} $\eps$-\emph{improbable goal reaching property} for $\eps \in [0,1)$ if there exists a function $\beta \in \KL$ such that for all initial states $\state_0 \in \states$:
	\[
		\PP{
			\goaldist(\State_t^{\policy}(\state_0)) \leq \beta(\goaldist(\state_0), t) \text{ for all } t
		}
		\ge
		1 - \eps
	\]
\end{dfn}

\section{PROPOSED APPROACH}\label{sec:approach}

\emph{CALF-Wrapper} offers a general methodology for strengthening RL policies through via an adaptive policy selection mechanism.
This approach dynamically switches between a performance-optimized base RL policy and a fallback controller that satisfies the $\eps$-improbable goal reaching property (Definition~\ref{dfn:eta_improbable}) by leveraging value function evaluations.
While maintaining the performance advantages of traditional RL, it simultaneously provides theoretical goal reaching guarantees which we prove in Theorems~\ref{thm:etagoalreaching} and~\ref{thm:unform_goal_reaching}.

In our setup (see Algorithm~\ref{alg_calfwrapper}), we assume access to:

(1) A \emph{base policy} $\basepolicy \in \policies_{\mathrm{stat}}$ parameterized by a deep neural network, learned via a standard RL procedure to optimize~\eqref{eq:rl_objective}, aided by a value function estimate $\Valuebase$.

(2) A \emph{fallback policy} \(\fallbackpolicy\in\policies_{\mathrm{stat}}\) meeting the $\eps$-improbable goal reaching property (Definition~\ref{dfn:eta_improbable}) for some $\eps$. The fallback policy need not be reward-optimal in terms of~\eqref{eq:rl_objective} and can be constructed using any appropriate method, e.g. PID or energy-based approaches can be used.

\subsection{Action Selection Mechanism and Its Benefits}

\subsubsection{Critic-Based Decision Rule}

The central innovation of CALF-Wrapper is its action selection strategy based on the critic value. We maintain a reference value $\Value_t^\dagger$, which tracks the \emph{best} critic value observed so far, and define a small improvement threshold $\nu>0$. At each time step $t$, the policy makes decisions as follows:

\noindent
\textit{Improvement Case.} If the current state's critic value $\Valuebase(\State_t)$ exceeds the reference value by at least $\nu$ (i.e., $\Valuebase(\State_t) \ge \Value_t^\dagger + \nu$), we update $\Value_{t+1}^\dagger := \Valuebase(\State_t)$ and select an action from the base policy: $\Action_t \sim \basepolicy(\bullet \mid \State_t)$.

\noindent
\textit{Non-Improvement Case.} Otherwise, we set $\Value_{t+1}^\dagger := \Value_t^\dagger$ and make a probabilistic decision:
\begin{itemize}
	\item With probability $\rho_{t}(\State_t)$, we still accept the base policy: $\Action_t \sim \basepolicy(\bullet \mid \State_t)$.
	\item With probability $1-\rho_{t}(\State_t)$, we switch to the fallback policy: $\Action_t \sim \fallbackpolicy(\bullet \mid \State_t)$.
\end{itemize}

The softened acceptance probability $\rho_{t}(\state)$ can be any sequence meeting the condition $\sum_{t=0}^{\infty} \sup_{\state \in \states} \rho_{t}(\state) < \infty$, which ensures that $\rho_{t}(\State_t)$ decays to zero as $t \to \infty$.

\subsubsection{Intuition Behind the Mechanism}

Because we treat $\Valuebase$ as an indicator of ongoing performance, we check whether the state's value remains above a running benchmark $\Value_t^\dagger$. If so, we accept the RL action; if it falls below, we revert to the conservative $\fallbackpolicy$.
The acceptance probability $\rho_{t}(\State_t)$ helps avoid rejecting $\basepolicy$ prematurely due to minor fluctuations in $\Valuebase$, acting like a simulated annealing mechanism that occasionally permits risk-taking for potential long-term gains.

\subsubsection{Adjustable Performance-Safety Balance}
In our implementation, we set
$
	\rho_{t}(\state) \equiv \lambda^t p_{\text{relax}},
$
where $\lambda \in (0,1)$ and $p_{\text{relax}} \in [0,1]$. Hyperparameter $p_{\text{relax}}$ represents our initial trust in the base policy, and $\lambda$ determines how quickly this trust decays over time. On finite horizons, setting $p_{\text{relax}} \approx 1$ and $\lambda \approx 1$ makes the behavior nearly identical to the original RL policy, preserving its performance advantages. By tuning $p_{\text{relax}}$ and $\lambda$, one can bias the policy toward $\basepolicy$ (thus recovering its baseline performance) or impose more frequent switches to $\fallbackpolicy$. In many cases, this approach can even outperform both $\basepolicy$ and $\fallbackpolicy$ in terms of cumulated reward (see Section~\ref{sec:experiments}).

\subsubsection{Seamless Integration}
Our approach is designed for seamless integration into existing RL pipelines. It is implemented as a standard Gymnasium environment wrapper~\cite{gym-wrapper} and operates as a non-intrusive wrapper around trained RL policies, requiring no modifications to the original training algorithm.

\subsubsection{Rigorous Stability Guarantees}
Theorem~\ref{thm:etagoalreaching} proves that CALF-Wrapper inherits the goal reaching properties from the fallback policy, ensuring system stability in worst-case scenarios. For uniform stabilizers, Theorem~\ref{thm:unform_goal_reaching} establishes precise bounds on both maximum overshoot and convergence time, with reaching time proven to be almost surely finite and following a precisely characterized distribution.

\subsection{Theoretical Analysis}

\begin{dfn}\label{dfn:calfwpolicy}
	Let $\calfwpolicy \in \policies_{\mathrm{nstat}}$ be the policy defined by  \Cref{alg_calfwrapper}. Specifically, for all $t \in \Z_{\geq0}$, set
	\begin{equation*}
		\calfwpolicy := \begin{cases}
			\basepolicy. \text{ if } \Valuebase(\State_t) \ge \Value_t^\dagger + \nu \text{ or } U_t < \rho_t(\State_t) \\
			\fallbackpolicy, \text{ otherwise.}
		\end{cases}
	\end{equation*}
\end{dfn}

\begin{theorem}\label{thm:etagoalreaching}
	Suppose that $\bar{\Value} = \sup_{\state \in \states} \Valuebase(\state) < \infty$. Then policy $\calfwpolicy$ satisfies the $\eps$-improbable goal reaching property from Definition~\ref{dfn:eta_improbable}.
\end{theorem}

\begin{proof}
	The key insight is that our algorithm will eventually use only the fallback policy, thus inheriting its goal-reaching property. We formalize this by showing that the total number of base-policy acceptances is almost surely finite.

	Let $N_{\basepolicy} := \sum_{t=0}^{\infty} \mathbf{1}\{\basepolicy=\calfwpolicy \text{ at time } t\}$ denote the total number of times the base policy is chosen. By Definition~\ref{dfn:calfwpolicy}, the base policy is chosen when either (i) the critic value improves significantly (i.e., $\Valuebase(\State_t) \geq \Value_t^\dagger + \nu$) or (ii) the algorithm randomly accepts it (i.e., $U_t < \rho_t(\State_t)$). Therefore:
	\begin{equation*}
		N_{\basepolicy} \leq N_{\Value^\dagger} + N_{\rho}
	\end{equation*}
	where $N_{\Value^\dagger} := \sum_{t=0}^{\infty} \mathbf{1}\{\Valuebase(\State_t) \geq \Value_t^\dagger + \nu\}$ counts critic improvements, and $N_{\rho} := \sum_{t=0}^{\infty} \mathbf{1}\{U_t < \rho_t(\State_t)\}$ counts random acceptances. 	We prove both terms are finite:

	\emph{Finiteness of $N_{\Value^\dagger}$}: Each time $\Valuebase(\State_t) \geq \Value_t^\dagger + \nu$ occurs, the value $\Value_t^\dagger$ increases by at least $\nu$. Since $\Valuebase$ is bounded above by $\bar{\Value}$, and $\Value_t^\dagger$ starts at $\Valuebase(\State_0)$, the value can increase only finitely many times. Thus, $N_{\Value^\dagger} < \infty$.

	\emph{Finiteness of $N_{\rho}$}: Since $\rho_t(\State_t) \leq \bar{\rho}_t$ for all $t$ and $\state$, and $\sum_{t=0}^\infty \bar{\rho}_t < \infty$ by assumption, the Borel-Cantelli lemma~\cite{billingsley1995probability} ensures that the event $U_t < \bar{\rho}_t$ occurs only finitely many times almost surely. Therefore, $N_{\rho} < \infty$ almost surely.

	Combining these results, we have $N_{\basepolicy} \leq N_{\Value^\dagger} + N_{\rho} < \infty$ almost surely. This means there exists an almost surely finite time $\mathcal{T}_{\text{final}}$ after which the algorithm exclusively uses the fallback policy:
	\begin{equation}
		\Action_t \sim \fallbackpolicy(\bullet \mid \State_t) = \calfwpolicy(\bullet \mid \State_t) \quad \text{for all } t \geq \mathcal{T}_{\text{final}}
	\end{equation}

	Since the trajectory evolves under $\fallbackpolicy$ beyond time $\mathcal{T}_{\text{final}}$, and $\fallbackpolicy$ satisfies the $\eps$-improbable goal-reaching property, we conclude that for all initial states $\state_0 \in \states$:
	\begin{equation}
		\PP{\goaldist(\State_t^{\calfwpolicy}(\state_0)) \xrightarrow{t \to \infty} 0} \geq 1 - \eps
	\end{equation}
	Thus, $\calfwpolicy$ inherits goal-reaching property from $\fallbackpolicy$.
\end{proof}

\begin{algorithm}[t]
	\caption{Critic as Lyapunov Function Wrapper}\label{alg_calfwrapper}
	\begin{algorithmic}[1]

		\REQUIRE\label{lst:requirements} \text{}
		\begin{itemize}[leftmargin=0.4em]
			\item $\fallbackpolicy \in \policies_{\mathrm{stat}}$: Fallback policy satisfying Definition~\ref{dfn:eta_improbable}.
			\item $\basepolicy\in \policies_{\mathrm{stat}}$: Base policy trained via an RL method.
			\item $\Valuebase$: Value function estimate trained alongside $\basepolicy$ via the RL method.
			\item $\nu > 0$: Minimum improvement threshold for updating the best observed critic value $\Value_t^\dagger$.
			\item $\{\rho_{t}(\state)\}_{t \geq 0}$: Softened acceptance probabilities.
			      We assume they are bounded by a summable majorant $\{\bar{\rho}_t\}_{t\ge 0}$, meaning $\rho_t(\state)\le\bar{\rho}_t \; \forall t,\state$ and $\sum_{t=0}^\infty \bar{\rho}_t<\infty$.
		\end{itemize}

		\STATE \textbf{Initialize}:
		\begin{itemize}
			\item $\State_0 \sim \transit_0(\bullet)$ or set $\State_0 = \state_0$ for some $\state_0 \in \states$
			\item $\Value_0^\dagger := \Valuebase(\State_0)$
		\end{itemize}

		\FOR{$t = 0,1,2,\dots$}
		\STATE\label{lst:u_t} $U_t \leftarrow$ sampled uniformly from $[0,1]$
		\IF{$\Valuebase(\State_t) \ge \Value_t^\dagger + \nu$ \textbf{or} $U_t < \rho_{t}(\State_t)$}
		\STATE\label{lst:line:accept_base} $\Action_t\la$ sampled from $\basepolicy(\bullet \mid \State_t)$
		\ELSE
		\STATE\label{lst:line:accept_fallback} $\Action_t\la$ sampled from  $\fallbackpolicy(\bullet \mid \State_t)$
		\ENDIF
		\STATE
		$
			\displaystyle
			\Value_{t+1}^\dagger := \begin{cases}
				\Valuebase(\State_t), \text{ if } \Valuebase(\State_t) \ge \Value_t^\dagger + \nu \\
				\Value_{t}^\dagger, \text{ otherwise }
			\end{cases}
		$
		\STATE $\State_{t+1} \la$ sampled from $\transit(\bullet \mid \State_t, \Action_t)$
		\ENDFOR
	\end{algorithmic}
\end{algorithm}

For Theorem~\ref{thm:unform_goal_reaching} we require the notion of \emph{a function with bounded superlevel sets}.
\begin{dfn}
	Let $X$ be a metric space, and let $f \colon X \to \R$.
	We say that $f$ is a \emph{function with bounded superlevel sets} if, for every $a\in f(X) = \{f(x) : x \in X\}$, the superlevel set
	$
		\{\,x \in X : f(x) \ge a\}
	$
	is bounded in $X$.
\end{dfn}
\begin{theorem}
	\label{thm:unform_goal_reaching}
	Consider \Cref{alg_calfwrapper}, initialized at $\state_0$ with $\goaldist(\state_0) \le d^{\circ}$,
	where $d^{\circ} \in \R_{>0}$ is arbitrary, and suppose that $\rho_{t}(\state) = 0$ whenever
	$\Valuebase(\state) < \Valuebase(\state_0)$.

	\noindent
	Assume additionally that:
	\begin{enumerate}[label=(\subscript{\mathrm{A}}{{\arabic*}}), leftmargin=2.5em]
		\item\label{ass:valuebase}
		$\Valuebase$ is continuous with bounded superlevel sets.
		\item\label{ass:uniform_fallback}
		$\fallbackpolicy$ satisfies the \emph{uniform $\eps$-improbable goal-reaching property}
		with certificate $\beta \in \KL$.
	\end{enumerate}

	\noindent
	Then the following claims hold:
	\begin{enumerate}[label=(\subscript{\mathrm{C}}{{\arabic*}}), leftmargin=2.5em]

		\item\label{claim:uniform_overshoot_bound}
		\textit{(\(\eps\)-improbable uniform overshoot boundedness)}
		There exists $\delta(d^{\circ})\in\R_{>0}$ such that
		\[
			\PP{\goaldist\!\bigl(\State_t^{\calfwpolicy}(\state_0)\bigr) \le \delta(d^{\circ})\,
				\text{ for all } t \ge 0}
			\ge 1-\eps.
		\]

		\item\label{claim:uniform_reaching_time}
		\textit{(\(\eps\)-improbable uniform reaching time)}
		For each $d^{*} \in (0, d^{\circ})$, there is an almost surely finite random time
		$T(d^{\circ},d^{*})\in\R_{\geq0}$ such that
		\[
			\PP{\goaldist\!\bigl(\State_t^{\calfwpolicy}(\state_0)\bigr) \le d^{*}
			\text{ for all } t \ge T(d^{\circ},d^{*})}
			\ge\ 1-\eps.
		\]

		\item\label{claim:distribution_reaching_time}
		\textit{(Reaching time distribution)}
		There exist natural numbers $\tau(d^{\circ})$ and $\tau_{f}(d^{\circ},d^{*})$
		such that for all $t \in \mathbb{Z}_{\ge 0}$,
		\[
			\PP{T(d^{\circ},d^{*}) \le (\tau(d^{\circ}) + t)\, \tau_{f}(d^{\circ},d^{*})}
			\!=\!
			\prod_{k=t}^{\infty}\bigl(1 - \bar{\rho}_{k}\bigr).
		\]
		where $\{\bar{\rho}_t\}_{t \geq 0}$ is defined in Require section of Algorithm~\ref{lst:requirements}.
		Moreover, $\prod_{k=t}^{\infty}\bigl(1 - \bar{\rho}_{k}\bigr) \ra 1$ as $t \ra \infty$.
	\end{enumerate}

	\noindent
	Furthermore, $\tau(d^{\circ})$, $\tau_{f}(d^{\circ},d^{*})$, and $\delta(d^{\circ})$ admit
	explicit formulas given in~\eqref{eq:T},~\eqref{eq:Tfallback} and~\eqref{eq:epsH}.
\end{theorem}

\begin{proof}
	The proof begins by introducing the definitions
	~\eqref{eq:vmin}--\eqref{eq:Tfallback} for the quantities
	\(\tau(d^{\circ})\), \(\tau_{f}(d^{\circ},d^{*})\), and \(\delta(d^{\circ})\).
	We then explain, step by step, how these definitions guarantee
	Claims~\ref{claim:uniform_overshoot_bound},~\ref{claim:uniform_reaching_time},
	and~\ref{claim:distribution_reaching_time}, making clear how each
	object is constructed from scratch.
	Specifically, define:
	\begin{align}
		v_{\min}(d^{\circ})
		 & := \min\{\Valuebase(\state) : \goaldist(\state)\le d^{\circ}\}
		\label{eq:vmin}                                                                   \\
		\mathbb{V}^{\basepolicy}(d^{\circ})
		 & := \{\state \in \states: \Valuebase(\state)\ge v_{\min}(d^{\circ})\}
		\label{eq:superset}                                                               \\
		v_{\max}(d^{\circ})
		 & := \max\{\Valuebase(\state) : \state \in \mathbb{V}^{\basepolicy}(d^{\circ})\}
		\label{eq:vmax}                                                                   \\
		\tau(d^{\circ})
		 & := \max\Bigl\{1,\,
		\Bigl\lfloor\dfrac{v_{\max}(d^{\circ}) - v_{\min}(d^{\circ})}{\nu}\Bigr\rfloor\Bigr\}
		\label{eq:T}                                                                      \\
		d_{\bar{\transit}}(d^{\circ})
		 & := \sup\{\bar{\transit}(s, a)
		: \state \in \mathbb{V}^{\basepolicy}(d^{\circ}),\,\action\in\actions\}
		\label{eq:Rrho}                                                                   \\
		d_{\max}(d^{\circ})
		 & := \max\bigl(d^{\circ},\, d_{\bar{\transit}}(d^{\circ})\bigr)
		\label{eq:R}                                                                      \\
		\delta(d^{\circ})
		 & := \beta(d_{\max}(d^{\circ}), 0).
		\label{eq:epsH}
	\end{align}

	In equations above $v_{\min}(d^{\circ})$ and $v_{\max}(d^{\circ})$ bound critic values, the superlevel set $\mathbb{V}^{\basepolicy}(d^{\circ})$ defines the operational domain of the base policy, $\tau(d^{\circ})$ limits best critic improvement iterations, while $d_{\bar{\transit}}(d^{\circ})$ bounds the maximum one-step transition from any state within the superlevel set $\mathbb{V}^{\basepolicy}(d^{\circ})$, and $d_{\max}(d^{\circ})$ provides a composite bound incorporating both initial distance and transition magnitudes.

	Further, we define $\tau_{f}(d^{\circ},d^{*})$, the minimum time steps for the fallback policy to drive any state with $\goaldist(\state) \leq d_{\max}(d^{\circ})$ to within $d^{*}$ of the goal:
	\begin{equation}
		\tau_{f}(d^{\circ},d^{*})
		:=
		\max\Bigl\{
		1,\,
		\Bigl\lceil -\log\Bigl(\xi^{-1}\left(\tfrac{d^{*}}{\kappa(d_{\max}(d^{\circ}))}\right)\Bigr)\Bigr\rceil
		\Bigr\},
		\label{eq:Tfallback}
	\end{equation}
	where $\kappa, \xi \in \K_{\infty}$ are functions such that $\beta(d,t) \leq \kappa(d)\xi(e^{-t})$ for all $d \geq 0, t \geq 0$ (a standard decomposition of $\mathcal{KL}$ functions, see~\cite[Lemma 8]{Sontag1998Comments}).

	Now we establish a series of results that together prove the theorem's claims:

	\begin{proposition}\label{prop:disallowed_superset}
		Whenever $\basepolicy = \calfwpolicy$ at time $t$, we have $\State_t \in \mathbb{V}^{\basepolicy}(d^{\circ})$.
	\end{proposition}
	\begin{whiteproof}
		By Definition~\ref{dfn:calfwpolicy}, the $\basepolicy$ is chosen in two cases:

		(1) When $\Valuebase(\State_t)\ge \Value^\dagger_t + \nu$: Since $\Value^\dagger_t$ never decreases and begins with $\Valuebase(\state_0)$, we have $\Valuebase(\State_t)\ge \Valuebase(\state_0)\ge v_{\min}(d^{\circ})$, which means $\State_t \in \mathbb{V}^{\basepolicy}(d^{\circ})$.

		(2) When $U_t < \rho_t(\State_t)$: By our assumption, $\rho_t(\state)=0$ whenever $\Valuebase(\state) < \Valuebase(\state_0)$. Therefore, if $\State_t \notin \mathbb{V}^{\basepolicy}(d^{\circ})$, then $\Valuebase(\State_t) < v_{\min}(d^{\circ}) \leq \Valuebase(\state_0)$, which means $\rho_t(\State_t)=0$. So this case can only occur when $\State_t \in \mathbb{V}^{\basepolicy}(d^{\circ})$.
	\end{whiteproof}

	\begin{proposition}\label{prop:T}
		The quantity $\tau(d^{\circ})$ in~\eqref{eq:T} bounds the number of times the critic value can significantly improve:
		\[
			\textstyle
			N_{\Value^{\dagger}} := \sum_{t=0}^{\infty} \mathbf{1}\{\Valuebase(\State_t) \ge \Value_t^\dagger + \nu\} \leq  \tau(d^{\circ})
		\]
	\end{proposition}
	\begin{whiteproof}
		The quantity $\tau(d^{\circ})$ is well-defined because $v_{\min}(d^{\circ})$ and $v_{\max}(d^{\circ})$ exist and are finite due to the compactness of the relevant sets and the continuity of $\Valuebase$ (Assumption~\ref{ass:valuebase}).

		Each time $\Valuebase(\State_t) \ge \Value_t^\dagger + \nu$ occurs, the value $\Value_t^\dagger$ increases by at least $\nu$. By Proposition~\ref{prop:disallowed_superset}, this can only happen when $\State_t \in \mathbb{V}^{\basepolicy}(d^{\circ})$, where $\Valuebase$ is bounded above by $v_{\max}(d^{\circ})$. Since $\Value_t^\dagger$ starts at $\Valuebase(\state_0) \geq v_{\min}(d^{\circ})$, it can increase at most $\lfloor(v_{\max}(d^{\circ}) - v_{\min}(d^{\circ}))/\nu \rfloor$ times, which is bounded by $\tau(d^{\circ})$.
	\end{whiteproof}

	\begin{proposition}\label{prop:Tq}
		Let
		\(
		\mathcal{T}_{\bar{\rho}} := \inf\{t \ge 0 : U_{k} \ge \bar{\rho}_{k} \;\forall k \ge t\}.
		\)

		(a) $\mathcal{T}_{\bar{\rho}}$ bounds the total number of random acceptances: $N_{\rho} := \sum_{t=0}^{\infty} \mathbf{1}\{U_t < \rho_t(\State_t)\} \leq \mathcal{T}_{\bar{\rho}}$.

		(b) $\mathcal{T}_{\bar{\rho}}$ is almost surely finite.

		(c) For all $t \in \mathbb{Z}_{\ge 0}$, $\PP{\mathcal{T}_{\bar{\rho}} \leq t} = \prod_{k=t}^{\infty}(1 - \bar{\rho}_k)$, and $\prod_{k=t}^{\infty}(1 - \bar{\rho}_k) \to 1$ as $t \to \infty$.
	\end{proposition}
	\begin{whiteproof}
		(a) Since $\rho_t(\state) \leq \bar{\rho}_t$ for all $t$ and $\state$, we have:
		\[
			\textstyle
			N_{\rho} = \sum_{t=0}^{\infty} \mathbf{1}\{U_t < \rho_{t}(\State_t)\} \leq \sum_{t=0}^{\infty} \mathbf{1}\{U_t < \bar{\rho}_t\} =: N_{\bar{\rho}}\]

		Furthermore, $N_{\bar{\rho}} \leq \sum_{t=0}^{\mathcal{T}_{\bar{\rho}}-1} 1 = \mathcal{T}_{\bar{\rho}}$, since by definition of $\mathcal{T}_{\bar{\rho}}$, the event $U_t < \bar{\rho}_t$ cannot occur for $t \geq \mathcal{T}_{\bar{\rho}}$.

		(b) The Borel-Cantelli lemma~\cite{billingsley1995probability} ensures that $N_{\bar{\rho}}$ is almost surely finite because $\sum_{t=0}^{\infty} \bar{\rho}_t < \infty$, which implies that $\mathcal{T}_{\bar{\rho}}$ is almost surely finite as well.

		(c) The event $\{\mathcal{T}_{\bar{\rho}} \leq t\}$ occurs if and only if $U_k \geq \bar{\rho}_k$ for all $k \geq t$. Since the $U_k$ are independent, we have:
		\[
			\textstyle
			\PP{\mathcal{T}_{\bar{\rho}} \leq t} = \prod_{k=t}^{\infty} \PP{U_k \geq \bar{\rho}_k} = \prod_{k=t}^{\infty} (1 - \bar{\rho}_k)\]

		Finally, $\prod_{k=t}^{\infty} (1 - \bar{\rho}_k) \to 1$ as $t \to \infty$ because $\sum_{k=t}^{\infty} \log(1 - \bar{\rho}_k) \to 0$ as $t \to \infty$, which follows from $\sum_{t=0}^{\infty} \bar{\rho}_t < \infty$.
	\end{whiteproof}


	\begin{proposition}\label{prop:next_state_invokefallback}
		If $\State_t \in \mathbb{V}^{\basepolicy}(d^{\circ})$ and $\Action_t \sim \basepolicy(\bullet \mid \State_t)$, then:

		(a) The next state satisfies $\|\State_{t+1}\| \leq d_{\bar{\transit}}(d^{\circ})$ almost surely.

		(b) Both $\state_0$ and all states with $\|\state\| \leq d_{\bar{\transit}}(d^{\circ})$ satisfy $\goaldist(\state) \leq d_{\max}(d^{\circ})$.
	\end{proposition}
	\begin{whiteproof}
		(a) By the definition of $\bar{p}$ (from our system assumption) and $d_{\bar{\transit}}(d^{\circ})$ in~\eqref{eq:Rrho}, when $\State_t \in \mathbb{V}^{\basepolicy}(d^{\circ})$ and $\Action_t \sim \basepolicy(\bullet \mid \State_t)$, we have $\|\State_{t+1}\| \leq \bar{p}(\State_t, \Action_t) \leq d_{\bar{\transit}}(d^{\circ})$ almost surely.

		(b) By definition, $\goaldist(\state_0) \leq d^{\circ} \leq d_{\max}(d^{\circ})$. For any state with $\|\state\| \leq d_{\bar{\transit}}(d^{\circ})$, we have $\goaldist(\state) \leq \|\state\| \leq d_{\bar{\transit}}(d^{\circ}) \leq d_{\max}(d^{\circ})$.
	\end{whiteproof}

	\begin{proposition}\label{prop:Tfallback}
		For any state $\state_0$ with $\goaldist(\state_0) \leq d_{\max}(d^{\circ})$:

		(a) $\PP{\goaldist(\State_t^{\fallbackpolicy}(\state_0)) \leq d^{*} \text{ for all } t \geq \tau_{f}(d^{\circ},d^{*})} \geq 1-\eps$.

		(b) $\PP{\goaldist(\State_t^{\fallbackpolicy}(\state_0)) \leq \delta(d^{\circ}) \text{ for all } t \geq 0} \geq 1-\eps$.
	\end{proposition}
	\begin{whiteproof}
		By Assumption~\ref{ass:uniform_fallback},
		\[
			\PP{\goaldist(\State_t^{\fallbackpolicy}(\state_0)) \leq \beta(\goaldist(\state_0),t) \text{ for all } t} \geq 1-\eps
		\]

		(a) From the definition of $\tau_{f}(d^{\circ},d^{*})$ in \eqref{eq:Tfallback}, we ensure that $\beta(\goaldist(\state_0),t) \leq d^{*}$ for all $t \geq \tau_{f}(d^{\circ},d^{*})$ when $\goaldist(\state_0) \leq d_{\max}(d^{\circ})$.

		(b) For any $t \geq 0$, we have $\beta(\goaldist(\state_0),t) \leq \beta(d_{\max}(d^{\circ}),0) = \delta(d^{\circ})$ when $\goaldist(\state_0) \leq d_{\max}(d^{\circ})$.
	\end{whiteproof}

	\textit{Conclusion the proof of Theorem \ref{thm:unform_goal_reaching}}. From Propositions~\ref{prop:T} and ~\ref{prop:Tq}, the total number of times the base policy is chosen is at most $\tau(d^{\circ}) + \mathcal{T}_{\bar{\rho}}$ almost surely.

	From Propositions~\ref{prop:disallowed_superset} and~\ref{prop:next_state_invokefallback}, whenever we switch to the fallback policy, the state satisfies $\goaldist(\state) \leq d_{\max}(d^{\circ})$.

	From Proposition~\ref{prop:Tfallback}, after running the fallback policy for $\tau_{f}(d^{\circ},d^{*})$ steps from any such state, we ensure with probability at least $1-\eps$ that the system stays within $d^{*}$ of the goal thereafter.

	Therefore, by defining the reaching time as:
	\begin{equation}
		\label{eq:Tcalfwpolicy}
		T(d^{\circ},d^{*}) := (\tau(d^{\circ}) + \mathcal{T}_{\bar{\rho}})\tau_{f}(d^{\circ},d^{*})
	\end{equation}
	We establish all three claims:

	Claim~\ref{claim:uniform_overshoot_bound} follows from Proposition~\ref{prop:Tfallback}(b), showing that the maximum deviation from the goal is bounded by $\delta(d^{\circ})$.

	Claim~\ref{claim:uniform_reaching_time} follows as the base policy is used at most $\tau(d^{\circ}) + \mathcal{T}_{\bar{\rho}}$ times, and $\tau_f(d^{\circ},d^{*})$ fallback steps suffice to maintain the system within $d^*$ of the goal with probability $1-\eps$, whether starting from the initial state or after base policy use.

	Claim~\ref{claim:distribution_reaching_time} follows directly from Proposition~\ref{prop:Tq}(c) and the definition of $T(d^{\circ},d^{*})$ in~\eqref{eq:Tcalfwpolicy}.
\end{proof}

\begin{corollary}
	If we set $\bar{\rho}_t = \lambda^t p_{\text{relax}}$, with $\lambda \in (0,1)$ and $p_{\text{relax}} \in [0,1]$,
	then for all $t \ge 1$ we have
	\begin{multline}\label{eq:lower_bound_prob}
		\PP{T(d^{\circ},d^{*}) \le (\tau(d^{\circ}) + t)\, \tau_{f}(d^{\circ},d^{*})} \\ \geq
		\exp\left(\tfrac{-\lambda^t p_{\text{relax}}}{(1 - \lambda)(1 - \lambda^t p_{\text{relax}})}\right).
	\end{multline}
\end{corollary}
\begin{proof}
	The probability from~\eqref{eq:lower_bound_prob} is:
	\[
		\prod_{k = t}^{\infty}\bigl(1 - \lambda^k\,p_{\text{relax}}\bigr)
		\;=\;
		\exp\left(\sum_{k = t}^{\infty}\log\bigl(1 - \lambda^k\,p_{\text{relax}}\bigr)\right).
	\]
	Note that $\log(1 - x)\ge\tfrac{-x}{1 - x}$ for $x \in [0,1)$,
	since the function $\log(1 - x)+\tfrac{x}{1 - x}$ is zero at $x=0$ and has nonpositive derivative on $[0,1)$. Thus, for $t\ge1$,
	\begin{equation*}
		\sum_{k = t}^{\infty}\log(1 - \lambda^kp_{\text{relax}}) \ge \sum_{k = t}^{\infty} \tfrac{-\lambda^kp_{\text{relax}}}{1 - \lambda^kp_{\text{relax}}} \ge
		\tfrac{-\lambda^t p_{\text{relax}}}{(1 - \lambda)(1 - \lambda^t p_{\text{relax}})},
	\end{equation*}
	which completes the proof.
\end{proof}
\section{ENVIRONMENTS}
\label{sec:environments}
\subsection{Pendulum-v1 Environment}
\begin{wrapfigure}{l}{0.2\linewidth}
	\includegraphics[width=\linewidth]{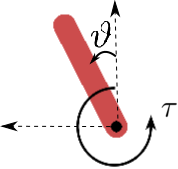}
\end{wrapfigure}

We use the \texttt{Pendulum-v1}~\cite{pendulum-v1} environment from Gymnasium, with dynamics:
\begin{equation*}
	\dot{\vartheta} = \omega \quad \dot{\omega} = -\frac{3g}{2l}\sin(\vartheta) + \frac{3}{ml^2}\tau
\end{equation*}
where $\vartheta$ is the angle, $\omega$ is angular velocity, and $\tau \in [-2,2]$ is the control torque (action). The initial state is randomly sampled from $(\vartheta, \omega) \sim \text{Uniform}([-\pi, \pi] \times [-1, 1])$. The observation is $[\cos(\vartheta), \sin(\vartheta), \omega]$. The goal is to stabilize the pendulum in the upright positition which corresponds to $\vartheta=0$. We define the goal set $\G$ as follows:
\begin{equation*}
	\G := \left\{\max\{|\cos \vartheta - 1|, |\sin\vartheta|\} \le\tfrac{1}{20} \land |\omega| \le \tfrac{3}{10}\right\}.
\end{equation*} 

The dynamics are integrated using the Euler scheme at 20 Hz. Episodes truncate after 200 steps.  The reward function is
\(
	r(\vartheta,\omega,\tau) = -\left(\arccos^2(\cos(\vartheta)) + 0.1\omega^2 + 0.001\tau^2\right).
\)
\subsection{CartPoleSwingup Environment}
The system dynamics are described by:
\begin{equation*}
	\ddot{x} = \tfrac{F + m_p l \dot{\vartheta}^2 \sin\vartheta - m_p g \sin\vartheta \cos\vartheta}{m_c + m_p\sin^2\vartheta}, \;
	\ddot{\vartheta} = \tfrac{g \sin\vartheta - \ddot{x}\cos\vartheta}{l},
\end{equation*}
where $x$ is the cart position, $\vartheta$ is the pole angle, $m_c=1.0$ kg is the cart mass, $m_p=0.1$ kg is the pole mass, $l=0.5$ m is the pole length, $g=9.8$ m/s$^2$ is the gravitational acceleration, and $F \in [-10, 10]$ N is the force applied to the cart (the control action).
\begin{wrapfigure}{l}{0.2\linewidth}
	\includegraphics[width=\linewidth]{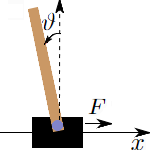}
\end{wrapfigure}
The initial state is randomly sampled from:
\((\vartheta, x, \dot{\vartheta}, \dot{x}) \sim \text{Uniform}\left([0, 2 \pi] \times [-1, 1]^3\right)\). The goal is to swing up and balance the pole in the upright position which corresponds to $\vartheta = 0$ while keeping the cart near $x = 0$. 
We define the goal set 
$\G$ as follows:
\begin{equation*}
	\left\{\max\{|\cos \vartheta - 1|, |\sin\vartheta|, |\dot{\vartheta}|\} \!\le\! \tfrac{1}{20} \land \max\{|\dot{x}|, |x|\} \!\le \!\tfrac{3}{10}\right\}.
\end{equation*}

The system is integrated using the Euler scheme at 50 Hz. Episodes truncate after 200 steps during training and 1000 steps during evaluation. Episodes terminate if either of the following conditions are met: $|x| > 5$, $|\dot{x}| > 8$, or $|\dot{\vartheta}| > 10$. The 5D observation is $[x, \dot{x}, \cos\vartheta, \sin\vartheta, \dot{\vartheta}]$ with reward function 
\[
r(x,\vartheta,\dot{x},\dot{\vartheta},F)\! = \!
	 - \tfrac{1}{2}\arccos^2(\cos\vartheta) - \tfrac{1}{2}x^2 - \tfrac{1}{20}\dot{\vartheta}^2 - \tfrac{1}{20}\dot{x}^2.
\]

\section{EXPERIMENTAL RESULTS}\label{sec:experiments}

For both environments, we implement fallback controllers (\(\fallbackpolicy\)) as the combination of energy-based and PD controllers.

We train base policies (\(\basepolicy\)) and their corresponding value functions (\(\Valuebase\)) using \emph{Proximal Policy Optimization} (PPO)~\cite{Schulman2017ProximalPolicy} from \texttt{Stable-Baselines3}~\cite{stable-baselines3}.
Training runs for 102k timesteps on Pendulum-v1 and 270k timesteps on CartPoleSwingup, with checkpoints saved at early, mid-stage, and late training phases.

We evaluate CALF-Wrapper using three operational modes by fixing $\lambda=0.9999$ and $\nu=0.01$ while varying $p_{\mathrm{relax}}$. (1)~\textit{Conservative} ($p_{\mathrm{relax}} = 0$): Prioritizes safety by defaulting to $\fallbackpolicy$ when uncertain. (2)~\textit{Balanced} ($p_{\mathrm{relax}} = 0.5$): Maintains equilibrium between exploration and safety. (3)~\textit{Brave} ($p_{\mathrm{relax}} = 0.95$): Maximizes exploitation of $\basepolicy$.

Performance assessment involves 30 independent trials from randomly sampled initial conditions, measuring both mean cumulative reward and goal-reaching success rate (percentage of trials reaching the goal set). 
The results are shown in \Cref{fig:calf_relax_sweeps}.

Our findings reveal clear performance-safety tradeoffs:

The \textit{Conservative Mode} achieved 100\% goal-reaching success across all experiments, consistently outperforming the standalone fallback policy. However, it delivers lower rewards compared to other modes during mid-stage and late training phases.

The \textit{Balanced Mode} demonstrates poor performance in early training phase but significantly improves in mid-stage and late phases, achieving both goal-reaching guarantees and higher cumulative rewards than the conservative approach.

The \textit{Brave Mode} closely tracks the base policy's reward performance throughout all phases. While this results in lower goal-reaching success during early and mid-stage phases, it achieves both optimal reward and 100\% goal-reaching success in the late phase, effectively combining the strengths of both policies.

Despite theoretical guarantees, for \textit{Balanced} and \textit{Brave} modes, goal-reaching rate occasionally falls below 100\% on early and mid-stage phases since the theoretically required number of steps to reach the goal set $\G$ exceeds the evaluation episode length.

\begin{figure}[h]
	\centering
	\includegraphics[width=\linewidth]{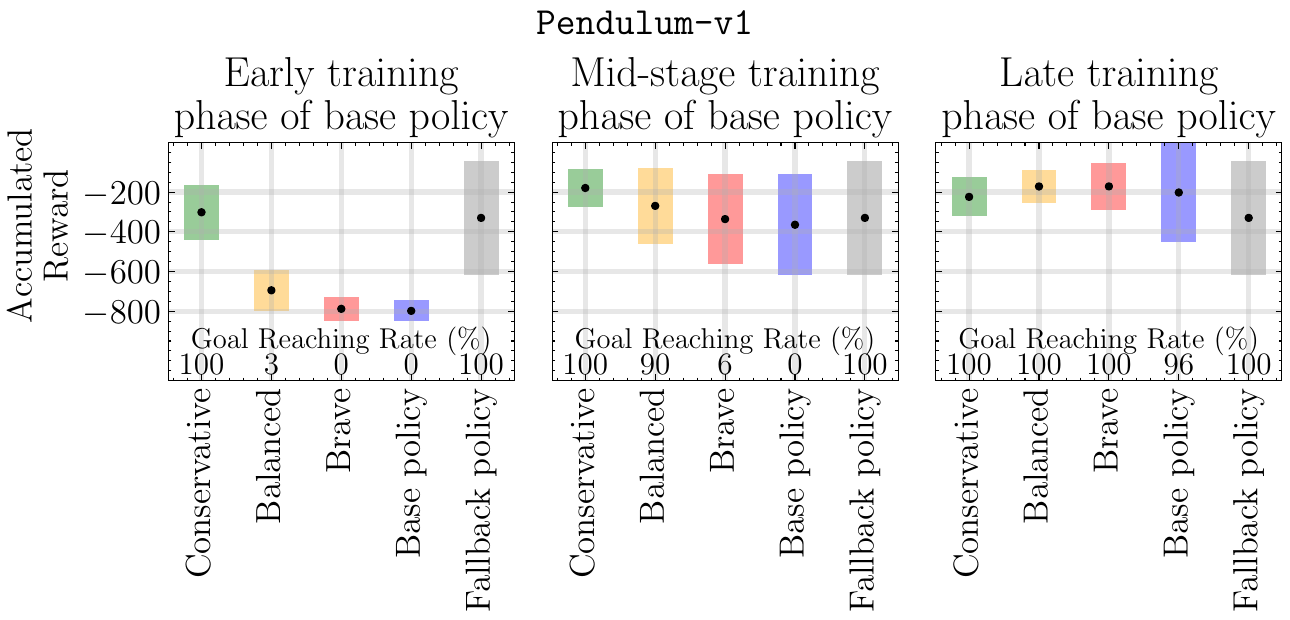}
	\includegraphics[width=\linewidth]{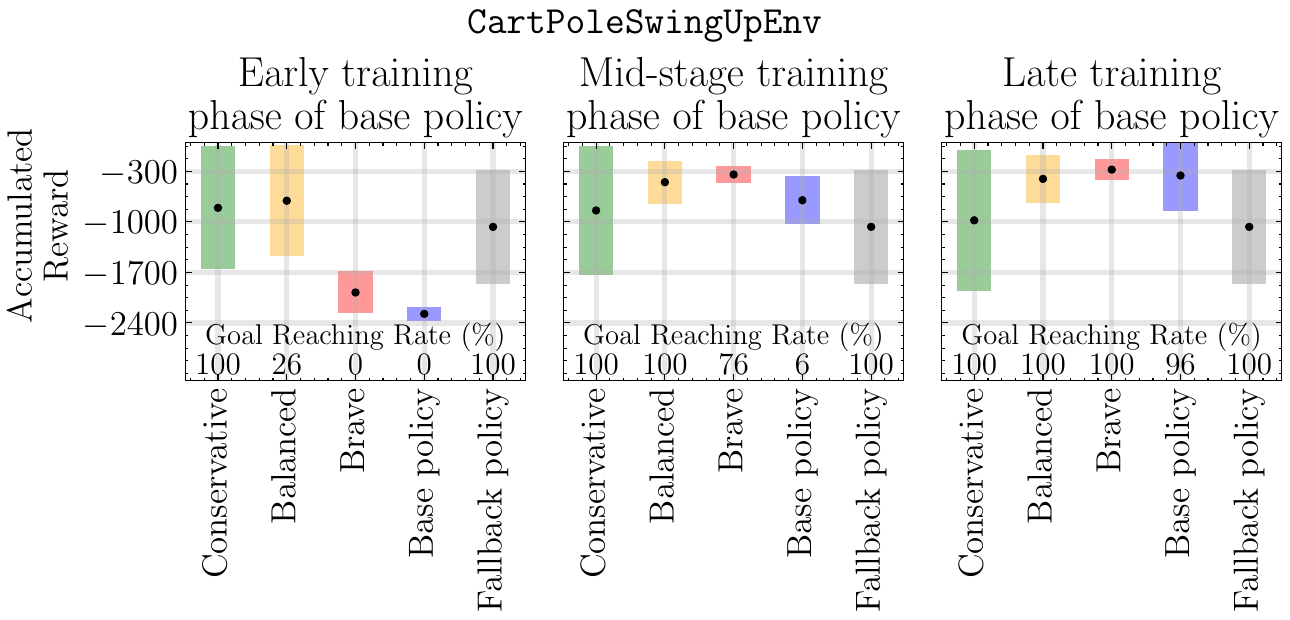}
	\caption{Comparative analysis of CALF-Wrapper modes on \texttt{Pendulum-v1} and \texttt{CartPoleSwingup} environments: Conservative ($p_{\mathrm{relax}} = 0$), Balanced ($p_{\mathrm{relax}} = 0.5$), and Brave ($p_{\mathrm{relax}} = 0.95$) modes against base and fallback policies. Box plots show mean cumulative rewards with standard deviations across 30 trials. }
	\label{fig:calf_relax_sweeps}
\end{figure}

\bibliographystyle{IEEEtran}
\bibliography{
	bib/CALFW.bib
}

\end{document}